\documentclass{article}
\usepackage{preprint,times}

\usepackage{hyperref}
\usepackage{url}
\usepackage{booktabs}
\usepackage{graphicx}

\usepackage{tabularx}
\usepackage{tcolorbox}

\usepackage{amsmath}
\usepackage{amssymb}
\usepackage{amsthm}

\usepackage{xspace}

\newtheorem*{theorem*}{Theorem}
\newtheorem{proposition}{Proposition}
\newtheorem*{proposition*}{Proposition}
\newtheorem{remark}{Remark}

\newcommand{\method}{{Harness Primitives\xspace}}
\newcommand{\ours}{{STITCH\xspace}}

\title{Composing Task-specific Agent Harnesses at Test Time with Reusable Primitives}

\author{
  Peng Kuang$^{1}$,
  Haibo Jin$^{1}$,
  Dehao Wu$^{1}$,
  Feiyang Deng$^{2}$,
  Xiaopeng Yuan$^{1}$,\\
  \textbf{Jerry Wang$^{1}$,}
  \textbf{Haohan Wang$^{1}$}\\
  $^1$University of Illinois Urbana-Champaign,
  $^2$University of Michigan, Ann Arbor \\
  \texttt{pengk2@illinois.edu,}
  \texttt{haohanw@illinois.edu}
}

\usepackage{enumitem}
\usepackage{subcaption}
\setlist[itemize]{noitemsep,leftmargin=*,topsep=0pt}
\setlist[enumerate]{noitemsep,leftmargin=*,topsep=0pt}

\iclrfinalcopy
\begin{document}

\maketitle

\begin{abstract}
Agent harnesses govern how large language models (LLMs) gather context, invoke tools, verify results, preserve state, and terminate, largely affecting agent performance. However, the value of each harness mechanism can differ across heterogeneous tasks: a mechanism that improves one task may impose overhead or context distraction on another, leading to the suboptimality of a global harness. We characterize this suboptimality as a mismatch induced by fixed mechanism choices, motivating task-specific harness construction. Nonetheless, generating harness code for each task introduces generation and debugging costs, with execution risks that can compound as more mechanisms are generated. To address those challenges, we introduce \textbf{\method}, reusable harness mechanisms with clear application scope and composition contract mined from failed task trajectories. Based on \method, we propose \textbf{\ours}, a framework that \textbf{S}elects suitable primitives given \textbf{T}ask \textbf{I}nformation and compiles them into \textbf{T}ask-spe\textbf{C}ific \textbf{H}arnesses at test time. This separation enables task-specific harnesses without generating or repairing mechanism code at test time. Extensive experiments demonstrate that \ours~not only improves harness adaptability and robustness, but also scales with the primitive library size, boosting task success rates by up to 12 points over fixed harness baselines, surpassing human-designed harnesses like Codex CLI while maintaining a minimal test-time harness composition overhead of only 2.7\%, 638 times more efficient than generating task-specific harnesses from scratch. Ultimately, our work demonstrates that building task-adaptive harnesses can be beneficial for completing diverse tasks and that building reusable primitives can be a promising path towards this goal.

\end{abstract}

\section{Introduction}
\label{sec:intro}
For large language model (LLM) agents, successful task execution depends not only on the capabilities of the underlying model, but also on the \emph{agent harness} that governs how those capabilities are used, where the task success rate can differ significantly when the same model is paired with different harnesses~\citep{lee2026metaharnessendtoendoptimizationmodel,lee2026recursiveharnessselfimprovement}. 
The harness determines how context is assembled, the interaction protocols between the actor and the environment, and when execution terminates. 
Modern agent harnesses therefore incorporate a range of mechanisms in those aspects that the model needs to complete the tasks. In practice, these mechanisms are typically organized within a shared harness design that is fixed across tasks~\citep{yang2026betterharnessessmallermodels,sengupta2026harborautomatedharnessoptimization}.

However, the effect of each mechanism can differ in heterogeneous tasks, leading to the suboptimality of a global harness \citep{zhang2026jitagentscalingharnessintelligence,chen2026harnessxcomposableadaptiveevolvable}. We characterize this suboptimality as a \textit{task–mechanism mismatch} induced by fixed mechanism choices: whenever a mechanism has positive utility within its application scope and incurs a penalty outside it, either fixed choice is strictly suboptimal to task-conditioned activation of that mechanism. (Section \ref{sec:theory}) This highlights the theoretical advantage of and the need for \emph{task-specific harness engineering}, in which the mechanisms governing execution are built according to the requirements of each task.

To realize this advantage, task-specific harness engineering faces a practical challenge: generating executable harness code on the fly can be brittle, introducing independent risks of syntax errors, logic flaws, and interface mismatches, along with generation and debugging costs \citep{zhang2026jitagentscalingharnessintelligence}. We further analyzed how these execution risks compound in Proposition \ref{thm:dominance}: under an independent mechanism-failure model, the probability of a generated harness remaining executable decays exponentially with a growing number of generated mechanisms. Thus, a task-adaptive harness that is not only flexible, but also robust to execution errors during test-time harness construction is needed.

Motivated by the challenge, we introduce \textbf{\method}, reusable harness mechanisms with clear application scope and composition contract. Specifically, we mine failed execution trajectories to identify recurring harness deficiencies, proposing and implementing new primitives to address each gap. Based on \method, we propose \textbf{\ours}~to compose task-specific harnesses at test time. \ours~includes a harness composer that evaluates the task instructions and execution contexts against the application scopes to select a subset of beneficial primitives called \textit{composition intent}, and a deterministic compiler then takes this composition intent and wires the selected primitives into an executable harness. This two-step composition separates task-conditioned primitive selection from dependency handling and executable graph construction, allowing the composer to focus on matching primitives to task requirements.

We evaluate \ours~on agent benchmarks including SWE-bench Verified~\citep{jimenez2024swebenchlanguagemodelsresolve} and Terminal-Bench 2~\citep{merrill2026terminalbenchbenchmarkingagentshard}. As a basis for comparison, the experiments test \ours~against baselines including Mini-SWE-agent~\citep{yang2024sweagentagentcomputerinterfacesenable}, Codex CLI, MemoHarness~\citep{huang2026memoharnessagentharnesseslearn}, and MetaHarness~\citep{lee2026metaharnessendtoendoptimizationmodel}. Results show that the success rate of \ours~scales with the size of the \method~library, outperforming fixed harness baselines by 12 points, including human-designed harnesses like Codex CLI. In terms of efficiency, the process of composition incurs an overhead of 2.7\% of the cost of execution. This overhead constitutes at least a reduction by a factor of 638 compared to generation of code from scratch. For reliability, tests indicate the system activates components at a rate of 100 percent, showing the compiled primitives are not only valid but genuinely taking effect during actor execution. Regarding generalization, \ours~functions across models of actors and domains of tasks, further supporting the scalability of \ours. In summary, our main contributions are as follows:
\begin{itemize}
\item \textbf{Mechanism-level analysis of task-specific harness construction.} We show that task-dependent control requirements induce a task-mechanism mismatch for fixed mechanism choices and that the probability of generating a valid harness can decay exponentially with more mechanisms.
\item \textbf{Primitive-based task-specific harness engineering.} We develop \method~from failed execution trajectories, assigning each primitive an application scope and a composition contract. Based on this library, we introduce \ours, a test-time framework that separates task-conditioned primitive selection from deterministic compilation into executable harnesses.
\item \textbf{Empirical validation.} We evaluate our framework on SWE-bench Verified and Terminal-Bench 2. Our experiments support our analysis and design, demonstrating that \ours~is scalable on \method~library size and improves task success rates by up to 12 points over the seed harness, surpassing fixed harness baselines including human-designed architectures like Codex CLI, while maintaining a test-time computational overhead of only 2.7\%, at least 638 times less than generating the harness from scratch.
\end{itemize}

\section{Related Work}
\label{sec:relatedworks}

\textbf{Agent Harnesses.} Agent harnesses organize reasoning, tool use, memory, and feedback around a language model~\citep{ning2026codeagentharness,kuang2026kvprmefficientprocessreward}. Early systems develop reasoning--action loops and reflection~\citep{yao2023reactsynergizingreasoningacting,shinn2023reflexionlanguageagentsverbal,ICLR2026_94bd1b8c}, alongside interfaces for repository editing and environment interaction~\citep{yang2024sweagentagentcomputerinterfacesenable,wang2025openhandsopenplatformai}. For task-dependent control, HarnessX selects among complete harnesses produced by editing processor code and configurations~\citep{chen2026harnessxcomposableadaptiveevolvable}, whereas \ours~selects individual validated primitives and compiles them for each task without changing their implementations.

\textbf{Harness Optimization.} Program and workflow search~\citep{hu2025automateddesignagenticsystems,zhang2025aflowautomatingagenticworkflow} and harness optimizers such as Self-Harness~\citep{zhang2026selfharnessharnessesimprove}, Meta-Harness~\citep{lee2026metaharnessendtoendoptimizationmodel}, and AutoSaddler~\citep{park2026autosaddlerautomaticharnessoptimization} edit and evaluate candidate harnesses to select a fixed harness for all tasks, whereas our approach retains validated mechanisms in a library and \ours~selects which to execute separately for each task. MemoHarness's ~\citep{huang2026memoharnessagentharnesseslearn} Codex-based adaptation and Recursive Harness Self-Improvement ~\citep{lee2026recursiveharnessselfimprovement} change textual instructions and rely on the agent to follow the requested workflow descriptions, whereas \ours~connects implemented operations so that harness code determines and enforces the selected mechanisms to be executed. Test-Time Harness Evolution ~\citep{nie2026tthetesttimeharnessevolution} rewrites Python harness code from test-batch traces, while JIT-Agent ~\citep{zhang2026jitagentscalingharnessintelligence} generates task-specific modules and repairs execution errors, both of which suffer from the vulnerability and cost of code generation. \ours~instead constructs each task's harness by selecting and connecting existing implementations, so test-time adaptation requires neither generating nor repairing mechanism code.

\section{Theoretical Analysis of Task-Adaptive Harnesses}
\label{sec:theory}

We analyze static harnesses, raw code generation, and primitive-based composition under a simplified model. Let $\mathcal{X}$ denote the space of tasks drawn from a distribution $\mathcal{D}$. An agent harness $H$ has task success probability $R(x,H)\in[0,1]$ conditional on valid execution. Let $V(H)\in\{0,1\}$ indicate the absence of the modeled implementation failures. We assign zero utility to invalid harnesses and write $J(x,H)=V(H)R(x,H)$. Expectations below include task sampling and any randomness in primitive selection and harness generation.

We formalize our primitive library as $\mathcal{P} = \{p_1, \dots, p_M\}$. Each primitive $p_i$ is characterized by an \textit{application scope} $\Omega_i \subseteq \mathcal{X}$ with task prevalence $\mu_i = \mathbb{P}_{x \sim \mathcal{D}}(x \in \Omega_i) \in (0, 1)$. For each task $x \in \Omega_i$, activating $p_i$ yields a marginal task utility $u_i(x) > 0$, with conditional mean $\bar{u}_i = \mathbb{E}[u_i(x) \mid x \in \Omega_i] > 0$. For each task $x \notin \Omega_i$, activating it incurs a task-utility penalty $c_i(x) > 0$, with conditional mean $\bar{c}_i = \mathbb{E}[c_i(x) \mid x \notin \Omega_i] > 0$. A composed harness is parameterized by an activation vector $\mathbf{a} \in \{0, 1\}^M$. Assuming valid execution ($V(H_{\mathbf{a}}) = 1$), the performance follows an additive structure: $R(x, \mathbf{a}) = R_0(x) + \sum_{i=1}^M a_i \Delta_i(x)$, where $R_0(x)$ is the baseline performance without primitives, and $\Delta_i(x) = u_i(x)$ if $x \in \Omega_i$, and $-c_i(x)$ otherwise. We assume all primitive combinations are admissible and the additive expression lies in $[0,1]$ for every task and activation vector, abstracting away interactions between primitives.

\begin{proposition}[Task-mechanism Mismatch of Fixed Activations]
\label{thm:suboptimal}
Assume $\mu_i \in (0,1)$, $\bar{u}_i > 0$, and $\bar{c}_i > 0$ for all $i \in \{1, \dots, M\}$. The expected performance of the optimal task-independent harness, $\mathbf{a}_{\text{fixed}}^* \in \arg\max_{\mathbf{a} \in \{0, 1\}^M} \mathbb{E}_{x \sim \mathcal{D}}[R(x, \mathbf{a})]$, is strictly lower than that of the oracle task-adaptive harness $\mathbf{a}^*(x)$. The suboptimality gap $\Gamma_{\text{fixed}}$ is defined as:
\begin{equation}
    \Gamma_{\text{fixed}} = \mathbb{E}_{x}[R(x, \mathbf{a}^*(x))] - \mathbb{E}_{x}[R(x, \mathbf{a}_{\text{fixed}}^*)] = \sum_{i=1}^M \min\left( \mu_i \bar{u}_i, \; (1 - \mu_i) \bar{c}_i \right) > 0
\end{equation}
\end{proposition}

\begin{remark}
Proposition~\ref{thm:suboptimal} quantifies the task-mechanism mismatch induced by fixed mechanism choices of a task-independent harness. Including a mechanism across tasks incurs an expected out-of-scope penalty of $(1-\mu_i)\bar{c}_i$, while excluding it sacrifices an expected in-scope utility of $\mu_i\bar{u}_i$. Under the stated model, task-conditioned choices avoid this fixed-choice trade-off.
\end{remark}

While Proposition~\ref{thm:suboptimal} motivates task-specific adaptation, raw code generation can introduce implementation failures. Let $k(x)=\sum_{i=1}^M\mathbf{1}(x\in\Omega_i)$ denote the number of oracle-selected mechanisms, and consider generating them in one attempt without repair. Conditional on task $x$, assume each mechanism fails independently with probability $\epsilon\in(0,1)$, and the harness is valid exactly when none of these failures occurs. Then $\mathbb{P}(V(H_{\text{gen}})=1\mid x)=(1-\epsilon)^{k(x)}$.

In contrast, a primitive-based harness selects validated primitives and composes them using a deterministic compiler. We assume the resulting compositions have valid execution, $V(H_{\text{comp}})=1$. With expectations include task sampling and any randomness in selection or generation, let $a_i$ denote the final activation of $p_i$, and define the utility-weighted false positive and false negative rates by
\begin{equation}
    \alpha_i =
    \frac{\mathbb{E}[a_i c_i(x)\mid x\notin\Omega_i]}{\bar{c}_i},
    \qquad
    \beta_i =
    \frac{\mathbb{E}[(1-a_i)u_i(x)\mid x\in\Omega_i]}{\bar{u}_i}.
\end{equation}

\begin{proposition}[Conditional Dominance of Primitive-Based Composition]
\label{thm:dominance}
Under the model in Section~\ref{sec:theory}, let $\bar{k}=\mathbb{E}_x[k(x)]=\sum_{i=1}^M\mu_i$. Primitive-based task-adaptive composition outperforms the specified alternatives under the following conditions:
\begin{enumerate}
\item \textbf{Dominance over Static Harnesses:} $\mathbb{E}[J(x,H_{\text{comp}})] >\mathbb{E}[J(x,H_{\text{fixed}}^*)]$ if and only if its utility-weighted selection loss is smaller than the fixed mismatch gap:
    \begin{equation}
        \sum_{i=1}^M
        \left[
            \mu_i\beta_i\bar{u}_i
            +(1-\mu_i)\alpha_i\bar{c}_i
        \right]
        <\Gamma_{\text{fixed}}.
    \end{equation}

\item \textbf{Dominance over Raw Code Generation:} For this part, assume $k(x)=\bar{k}$ for every task, although the selected mechanisms may differ across tasks. Suppose valid generated harnesses attain the oracle performance $R(x,\mathbf{a}^*(x))$. If $\mathbb{E}[J(x,H_{\text{comp}})]>0$ and with a small $\epsilon$, composition strictly outperforms this oracle generator if and only if,
    \begin{equation}
        \bar{k}>k^*
        =
        \frac{
            \ln\left(
                \frac{\mathbb{E}[R(x,\mathbf{a}^*(x))]}
                     {\mathbb{E}[J(x,H_{\text{comp}})]}
            \right)
        }{-\ln(1-\epsilon)}
        \approx
        \frac{1}{\epsilon}
        \ln\left(
            \frac{\mathbb{E}[R(x,\mathbf{a}^*(x))]}
                 {\mathbb{E}[J(x,H_{\text{comp}})]}
        \right).
    \end{equation}
\end{enumerate}
\end{proposition}

\begin{remark}
Part 1 shows that the composer need not be perfect: its selection loss need only remain below the mismatch gap of the best task-independent configuration. Part 2 characterizes a reliability trade-off under constant mechanism count and independent implementation failures. These conditional results motivate separating task-conditioned selection from implementation generation; they do not establish compiler correctness or superiority over code generation with repair.
\end{remark}

\section{Methodology}
\label{sec:method}

Our approach separates the development of \method~from its test-time composition by \ours. As shown in Figure~\ref{fig:overview}, primitive development converts recurring execution failures into a library of implemented and validated mechanisms, each with an application scope and a composition contract (Section~\ref{sec:development}). Given this library and task information, \ours~uses a harness composer to select primitives and a deterministic compiler to connect them into a task-specific harness (Section~\ref{sec:composition}). This separation allows mechanism choices to vary across tasks while reusing their implementations.

\begin{figure}[t]
    \centering
    \includegraphics[width=1\linewidth]{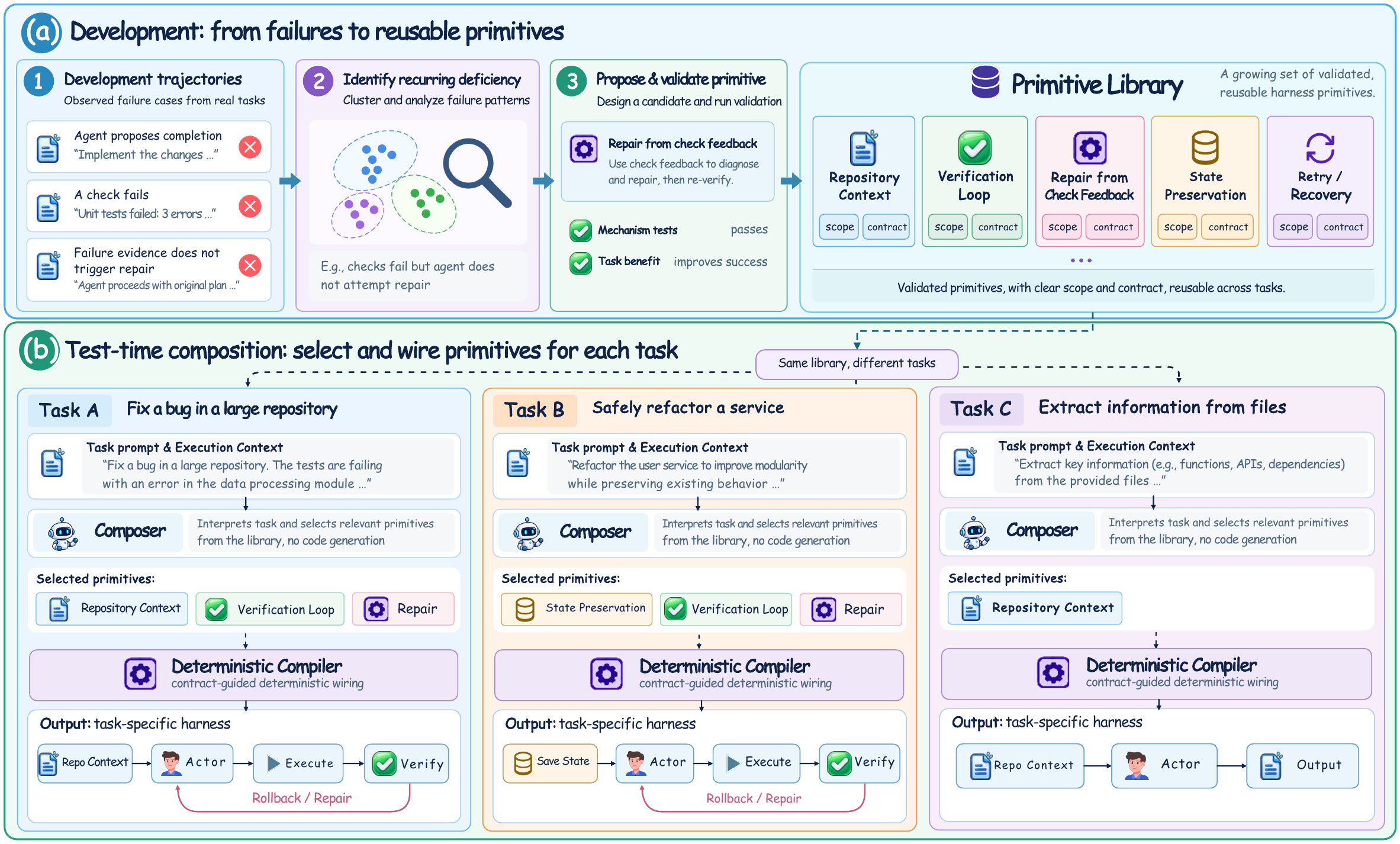}
    \vspace{-5mm}
\caption{Overview of \ours. Development stage turns recurring failures into a library of implemented primitives, and test-time harness composition selects and connects these primitives for individual tasks.}

    \label{fig:overview}
\end{figure}

\subsection{Preliminaries}
\label{sec:problem}

\textbf{Agent Harness.} Let $x$ denote a task instruction and $c_x$ an execution specification describing the task-solving language model, its interaction interface, and the task environment setup. We call the task-solving model the \emph{actor}. A \emph{harness} $H$ is a control program that mediates the actor's interaction with the environment: it constructs model requests, executes actions through available tools, returns observations, and determines when to terminate. Together, the actor, harness, and environment produce an execution trajectory $\tau$ for success evaluation.

\textbf{Task-specific harness construction.} The problem is to design a construction procedure $F$ that maps the available task information to a harness $H_x=F(x,c_x)$. The objective is to improve the probability of successful task completion while holding the actor model and environment fixed. Unlike using a shared harness across tasks, this formulation allows the harness program to depend on the individual task. Let $\mathcal{H}(c_x)$ be the set of compatible harnesses, $p(\cdot\mid x,c_x,H)$ the induced trajectory distribution, and $R_x(\tau)\in\{0,1\}$ indicate task success. The goal is to find a task-specific harness $H_x^\star$ that satisfies:
\begin{equation}
H_x^\star \in
\underset{H\in\mathcal{H}(c_x)}{\arg\max}\;
\mathbb{E}_{\tau\sim p(\cdot\mid x,c_x,H)}
\bigl[R_x(\tau)\bigr].
\label{eq:harness-objective}
\end{equation}

\textbf{Primitives and harness composition.} We instantiate harness construction through a library $\mathcal{P}$ of \emph{primitives}: implemented, reusable control operations, each with an application scope and a composition contract. The scope describes when the primitive is expected to help, while the contract specifies its inputs, outputs, environment requirements, and dependencies. We represent a harness $H$ as a directed graph $G=(V,E)$, where $V$ is the set of control-operation nodes and $E$ is the set of directed edges specifying execution order and conditional branches. Each operation node invokes a primitive, and operations communicate through shared execution state containing the interaction history and intermediate results. This representation allows the same operations to be reused in different task-specific control flows. Details in Appendix \ref{app:primitive-example}.

\subsection{Primitive Development}
\label{sec:development}

The Primitive Harness is initialized with a development set $\mathcal{D}_{\mathrm{dev}}$ and a \emph{seed harness} $H_{\mathrm{seed}}$. The development process starts with failure-motivated primitive propositions. Then, primitive propositions are implemented, verified, and revised until a library of useful primitives is built. Finally, each accepted primitive is assigned an application scope and composition contract to guide the primitive selection and harness composition at test time, as detailed in Section \ref{sec:composition}.

\textbf{Primitive proposal: From failures to reusable mechanisms.} Primitive development takes a development set $\mathcal{D}_{\mathrm{dev}}$ and a seed harness $H_{\mathrm{seed}}$ as input for collecting failed trajectories. A \emph{development agent}, an LLM agent responsible for proposing, implementing, and evaluating reusable control mechanisms, runs the actor with $H_{\mathrm{seed}}$ on tasks in $\mathcal{D}_{\mathrm{dev}}$ and collects the resulting trajectories $\tau$ and task outcomes $R_x(\tau)$. The development agent is distinct in role from the task-solving actor. It analyzes failed attempts, for which $R_x(\tau)=0$, to identify recurring deficiencies. For each deficiency, the development agent proposes a mechanism, specifies when it should activate and which subsequent actor decision it should influence, and implements the corresponding primitives.

\textbf{Primitive validation and revision.} Our validation protocol separates mechanism correctness from task benefit. Mechanism tests check whether the intended primitive operation activates and whether its output changes the relevant request, action, or control-flow decision. Matched development runs then compare the trajectory of a harness with and without the candidate primitive to assess its benefit in task success $R_x(\tau)$. The evidence informs iterative revision and library acceptance. Only primitives with a positive effect on the performance are accepted, and the process terminates when the number of valid primitives within the library passes a given threshold $\theta$.

\textbf{Scope and composition contracts.} For each accepted primitive, we assign an \emph{application scope}, describing when it is expected to facilitate the task and when its overhead or behavior may be undesirable according to its success on the evaluated tasks, and a \emph{composition contract}, specifying its inputs, outputs, required environment setup, and dependencies. Scope descriptions guide the test-time harness composer, which selects primitives, while contracts guide the implementation of the harness compiler, which connects them into a harness.

\subsection{Test-Time Harness Composition}
\label{sec:composition}

After development, the accepted primitive library $\mathcal{P}$, application scopes, and composition contracts are fixed for test-time use. Given a task $x$ and its execution specification $c_x$, a harness composer selects suitable primitives, and a deterministic harness compiler connects them into an executable harness. This process instantiates the construction procedure $F$ in Section~\ref{sec:problem}. We denote the abstraction of the seed harness $H_{\mathrm{seed}}$ as a graph $G_{\mathrm{seed}}$ and the resulting task-specific graph by $G_x$, which represents $H_x=F(x,c_x)$.

\textbf{Primitive selection.} Before the first actor request, the harness composer $\pi$ with the same backbone LLM as the actor, uses the application scopes to match the task's demands to the accepted primitives and their contracts to account for environment requirements and dependencies. Specifically, the composer receives $x$, $c_x$, $G_{\mathrm{seed}}$, and the library descriptions. It outputs a \emph{composition intent} $I_x$, specifying which primitives to use at a design level rather than generating their implementation code from scratch or handling the integration of the primitives, which could be brittle as well. Example schema and implementation details can be found in Appendix \ref{app:composition-intent} and Appendix \ref{app:composer-prompt}.

\textbf{Contract-guided compilation.} Let $K$ denote the deterministic compiler and $\widetilde{G}_x$ the candidate graph it constructs from the intent. Selection and compilation are given by
\begin{equation}
\begin{aligned}
I_x = \pi(x,c_x;G_{\mathrm{seed}},\mathcal{P}), ~~\widetilde{G}_x = K(G_{\mathrm{seed}},I_x;\mathcal{P},c_x)
\end{aligned}
\label{eq:intent-compilation}
\end{equation}
The compiler connects the selected primitives to the seed graph and supplies the dependencies required by their composition contracts. For example, selecting a repository map, which summarizes task-relevant files and symbols, also requires inserting that map into the actor's request. The compiler supplies this connection without requiring the composer to specify it. The candidate composition intent is checked for supported primitives, compatibility with $c_x$, valid connections, reachable termination, and bounded loops. If considered invalid, the diagnostics are returned to the composer for bounded correction. If correction fails, the system retains the seed harness, $G_{\mathrm{seed}}$. Finally, the compiler builds the task-specific harness according to the valid graph.

\section{Experiments}

\subsection{Setup}

\textbf{Benchmark and metrics.} We evaluate on SWE-bench Verified~\citep{jimenez2024swebenchlanguagemodelsresolve}, which requires repairing real repository issues, and Terminal-Bench 2~\citep{merrill2026terminalbenchbenchmarkingagentshard}, which covers tasks performed in a terminal environment. The main tables report 100 randomly sampled SWE-bench tasks, grouped into Django, SymPy, and other repositories, and 45 randomly sampled Terminal-Bench tasks, grouped by difficulty. With $k=2$ attempts per task, we report Pass@1, the fraction of successful attempts; Pass@$k$, the fraction of tasks solved in at least one attempt~\citep{chen2021evaluatinglargelanguagemodels}; and Pass\textasciicircum{k}, the fraction solved in all $k$ attempts.

\textbf{Implementation details.} For \ours~and the harness optimization baselines, MemoHarness~\citep{huang2026memoharnessagentharnesseslearn} and MetaHarness~\citep{lee2026metaharnessendtoendoptimizationmodel}, we use Mini-SWE-agent 2.4.6~\citep{yang2024sweagentagentcomputerinterfacesenable} as the common seed harness and GPT-5.6-Sol as the development model. GPT-5.6-Luna with medium reasoning effort serves as the task-solving actor in the main experiments and as the primitive selector for \ours. In our analysis, we also use GPT-5.6-Terra with medium reasoning, DeepSeek-V4-Flash with no reasoning, and Claude-4.5-Haiku as out-of-domain actors models not used during development. We compare against the fixed Mini-SWE-agent, SWE-agent~\citep{yang2024sweagentagentcomputerinterfacesenable}, and Codex CLI harnesses, the two harness optimization baselines, and harnesses that apply fixed primitives across all tasks (Appendix \ref{app:primitive-catalog}). The development task set is sampled randomly, 100 from SWE-bench and 44 from Terminal-bench, both disjoint from the evaluation samples.

\subsection{Main Results}

\textbf{Task-adaptive composition improves task success.} Tables~\ref{tab:swe-luna} and~\ref{tab:terminal-random-test} show that \ours~leads all compared methods on all three overall metrics. These results are consistent with the motivation for task-conditioned mechanism selection in Section~\ref{sec:theory}. On SWE-bench, it achieves 80.5\% Pass@1, exceeding Mini-SWE-agent by 7.5 percentage points, Codex CLI by 1.5 points, and the stronger harness optimization baseline, MetaHarness, by 7.0 points. On Terminal-Bench, it reaches 72.2\% Pass@1, improving over these baselines by 12.2, 15.5, and 16.6 points, respectively. The gains extend to repeated success: Pass$^2$ rises from Mini-SWE-agent's 68.0\% to 76.0\% on SWE-bench and from 51.1\% to 66.7\% on Terminal-Bench. Furthermore, the results also show that task-adaptive selection exploits complementary primitives. \ours~exceeds the strongest fixed primitive's overall Pass@1 by 4.5 points on SWE-bench and 5.5 points on Terminal-Bench, while the random composer baseline barely reaches the average performance of fixed primitives. Individual primitives have distinct strengths, which echoes our assumptions in Section \ref{sec:theory}: Contract-case Explorer performs best among fixed primitives on SymPy and other repositories, while Environment-capability Runner leads on Django. On Terminal-Bench, State Guard leads on medium tasks, whereas Execution Supervisor leads on hard tasks. \ours~combines this coverage, reaching 81.0\% Pass@1 on medium tasks and 46.2\% on hard tasks while retaining 100\% on easy tasks. These patterns support the value of matching control mechanisms to individual tasks.

\begin{table}[htbp]
\centering
\caption{The performance of the GPT-5.6-Luna on SWE-bench Verified.}
\label{tab:swe-luna}
\small
\setlength{\tabcolsep}{2.5pt}
\renewcommand{\arraystretch}{1.12}
\resizebox{\linewidth}{!}{%
\begin{tabular}{@{}l*{12}{r}@{}}
\toprule
Method & \multicolumn{3}{c}{Overall} & \multicolumn{3}{c}{Django} & \multicolumn{3}{c}{SymPy} & \multicolumn{3}{c}{Other repos} \\
\cmidrule(lr){2-4}\cmidrule(lr){5-7}\cmidrule(lr){8-10}\cmidrule(l){11-13}
& Pass@1 & Pass@$k$ & Pass\textasciicircum{k} & Pass@1 & Pass@$k$ & Pass\textasciicircum{k} & Pass@1 & Pass@$k$ & Pass\textasciicircum{k} & Pass@1 & Pass@$k$ & Pass\textasciicircum{k} \\
\midrule
Mini-SWE-agent (Seed Harness) & 73.0 & 78.0 & 68.0 & 75.0 & 78.8 & 71.2 & 75.0 & 78.6 & 71.4 & 69.1 & 76.5 & 61.8 \\
SWE-agent & 47.5 & 63.0 & 32.0 & 53.8 & 71.2 & 36.5 & 42.9 & 50.0 & 35.7 & 39.7 & 55.9 & 23.5 \\
Codex CLI & 79.0 & 83.0 & 75.0 & 78.8 & 80.8 & \textbf{76.9} & 64.3 & 71.4 & 57.1 & \textbf{85.3} & \textbf{91.2} & 79.4 \\
MemoHarness & 70.5 & 77.0 & 64.0 & 70.2 & 76.9 & 63.5 & 67.9 & 71.4 & 64.3 & 72.1 & 79.4 & 64.7 \\
MetaHarness & 73.5 & 79.0 & 68.0 & 74.0 & 78.8 & 69.2 & 64.3 & 71.4 & 57.1 & 76.5 & 82.4 & 70.6 \\
\midrule
\multicolumn{13}{l}{\textit{Our developed primitives (fixed across tasks)}} \\
Change-surface Tracer & 73.5 & 77.0 & 70.0 & 76.9 & 80.8 & 73.1 & 60.7 & 64.3 & 57.1 & 73.5 & 76.5 & 70.6 \\
Compatibility-envelope Gate & 72.0 & 77.0 & 67.0 & 76.9 & 80.8 & 73.1 & 57.1 & 64.3 & 50.0 & 70.6 & 76.5 & 64.7 \\
Contract-case Explorer & 72.5 & 78.0 & 67.0 & 69.2 & 76.9 & 61.5 & 71.4 & 78.6 & 64.3 & 77.9 & 79.4 & 76.5 \\
Discriminating-oracle Runner & 71.5 & 76.0 & 67.0 & 75.0 & 78.8 & 71.2 & 57.1 & 57.1 & 57.1 & 72.1 & 79.4 & 64.7 \\
Domain-trace Template & 73.0 & 77.0 & 69.0 & 76.0 & 80.8 & 71.2 & 60.7 & 64.3 & 57.1 & 73.5 & 76.5 & 70.6 \\
Environment-capability Runner & 76.0 & 79.0 & 73.0 & 77.9 & 80.8 & 75.0 & 67.9 & 71.4 & 64.3 & 76.5 & 79.4 & 73.5 \\
\midrule
Random composer & 73.3 & 78.0 & 68.6 & 75.4 & 79.6 & 71.2 & 64.3 & 69.6 & 58.9 & 73.9 & 79.0 & 68.8 \\
\textbf{\ours~(Ours)} & \textbf{80.5} & \textbf{85.0} & \textbf{76.0} & \textbf{79.8} & \textbf{84.6} & 75.0 & \textbf{75.0} & \textbf{78.6} & \textbf{71.4} & 83.8 & 88.2 & \textbf{79.4} \\
\bottomrule
\end{tabular}%
}
\end{table}

\begin{table}[htbp]
\centering
\caption{The performance of the GPT-5.6-Luna on Terminal-Bench 2.}
\label{tab:terminal-random-test}
\small
\setlength{\tabcolsep}{2.5pt}
\renewcommand{\arraystretch}{1.12}
\resizebox{\linewidth}{!}{%
\begin{tabular}{@{}l*{12}{r}@{}}
\toprule
Method & \multicolumn{3}{c}{Overall} & \multicolumn{3}{c}{Easy} & \multicolumn{3}{c}{Medium} & \multicolumn{3}{c}{Hard} \\
\cmidrule(lr){2-4}\cmidrule(lr){5-7}\cmidrule(lr){8-10}\cmidrule(l){11-13}
& Pass@1 & Pass@$k$ & Pass\textasciicircum{k} & Pass@1 & Pass@$k$ & Pass\textasciicircum{k} & Pass@1 & Pass@$k$ & Pass\textasciicircum{k} & Pass@1 & Pass@$k$ & Pass\textasciicircum{k} \\
\midrule
Mini-SWE-agent (Seed Harness) & 60.0 & 68.9 & 51.1 & 100.0 & 100.0 & 100.0 & 65.5 & 75.9 & 55.2 & 38.5 & 46.2 & 30.8 \\
SWE-agent & 13.3 & 17.8 & 8.9 & 33.3 & 33.3 & 33.3 & 13.8 & 17.2 & 10.3 & 7.7 & 15.4 & 0.0 \\
Codex CLI & 56.7 & 66.7 & 46.7 & 66.7 & 66.7 & 66.7 & 67.2 & 79.3 & 55.2 & 30.8 & 38.5 & 23.1 \\
MemoHarness & 53.3 & 66.7 & 40.0 & 83.3 & 100.0 & 66.7 & 62.1 & 75.9 & 48.3 & 26.9 & 38.5 & 15.4 \\
MetaHarness & 55.6 & 66.7 & 44.4 & 66.7 & 100.0 & 33.3 & 62.1 & 72.4 & 51.7 & 38.5 & 46.2 & 30.8 \\
\midrule
\multicolumn{13}{l}{\textit{Our developed primitives (fixed across tasks)}} \\
State Guard & 66.7 & 75.6 & 57.8 & 83.3 & 100.0 & 66.7 & 77.6 & 86.2 & 69.0 & 38.5 & 46.2 & 30.8 \\
Check Replay & 65.6 & 75.6 & 55.6 & 100.0 & 100.0 & 100.0 & 74.1 & 86.2 & 62.1 & 38.5 & 46.2 & 30.8 \\
First-Failure Localizer & 62.2 & 68.9 & 55.6 & 100.0 & 100.0 & 100.0 & 72.4 & 82.8 & 62.1 & 30.8 & 30.8 & 30.8 \\
Experiment Keeper & 58.9 & 68.9 & 48.9 & 83.3 & 100.0 & 66.7 & 65.5 & 75.9 & 55.2 & 38.5 & 46.2 & 30.8 \\
Execution Supervisor & 65.6 & 75.6 & 55.6 & 100.0 & 100.0 & 100.0 & 70.7 & 82.8 & 58.6 & 46.2 & \textbf{53.8} & 38.5 \\
\midrule
Random composer & 63.5 & 71.9 & 55.2 & 94.4 & 100.0 & 88.9 & 71.3 & 81.0 & 61.5 & 39.1 & 44.9 & 33.3 \\
\textbf{\ours~(ours)} & \textbf{72.2} & \textbf{77.8} & \textbf{66.7} & \textbf{100.0} & \textbf{100.0} & \textbf{100.0} & \textbf{81.0} & \textbf{89.7} & \textbf{72.4} & \textbf{46.2} & 46.2 & \textbf{46.2} \\
\bottomrule
\end{tabular}%
}
\end{table}

\textbf{\ours~scales with primitive library size.} We explore how the size of the primitive library affects the performance of \ours~and how scalable it is to the primitive library size. Starting from the seed harness, we gradually increase the number of primitives in the primitive library used by \ours~at test time. As shown in Figure \ref{fig:scaling}, the performance of \ours~increases steadily from 75.0\% to 80.5\% on SWE-bench and from 62.2\% to 72.2\% on Terminal-Bench as the number of primitives in the library grows, showing strong scalability of the proposed primitive-based task-adaptive harnesses.

\begin{figure}[t]
\centering
\begin{subfigure}[t]{0.45\linewidth}
    \centering
    \includegraphics[width=\linewidth]{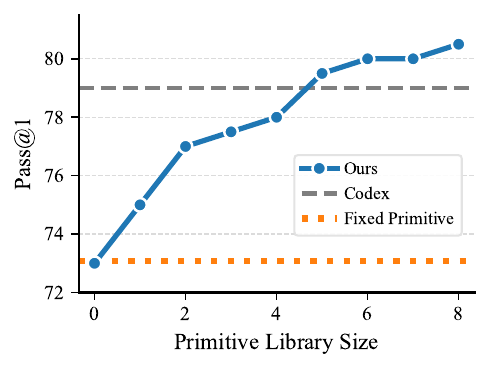}
\end{subfigure}
\begin{subfigure}[t]{0.45\linewidth}
    \centering
    \includegraphics[width=\linewidth]{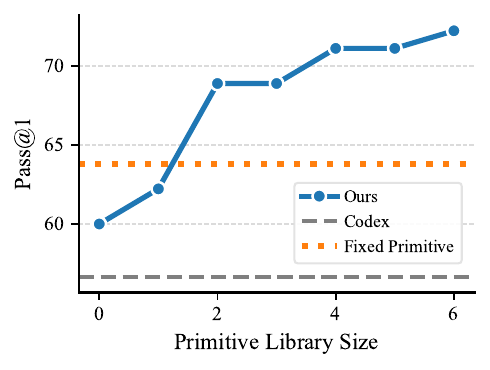}
\end{subfigure}
\caption{The scalability of \ours~on \method~Library size. \emph{Left}: The performance of \ours~on SWE-bench Verified scales with \method~Library size. \emph{Right}: The performance of \ours~on Terminal-bench 2 scales with \method~Library size.}
\label{fig:scaling}
\end{figure}

\subsection{Analysis}

\textbf{\method~generalizes across actor models.} To examine the cross-model generalization of \ours, we use out-of-domain actor models such as GPT-5.6-Terra, DeepSee-v4-Flash, and Claude-4.5-Haiku at test time, while GPT-5.6-Luna is used as the actor model during the development stage. The primitives implementation and composition contract remain unchanged while the application scope is adjusted according to the model's execution on development tasks. As shown in Table \ref{tab:cross-model}, the primitives generalize to different actor models, with the overall performance of all fixed single-primitive harnesses surpassing the seed harness. This shows that developed primitives represent generic harness control mechanisms instead of overfitting to the weaknesses of the actor model during development. Furthermore, the composer is also able to effectively select suitable primitives for the tasks, achieving the highest average performance across the evaluated actor models.

\begin{table}[htbp]
\centering
\small
\caption{Out-of-development-model actor generalization on SWE-bench Verified.}
\label{tab:cross-model}
\setlength{\tabcolsep}{2.5pt}
\renewcommand{\arraystretch}{1.12}
\resizebox{\linewidth}{!}{%
\begin{tabular}{@{}l*{12}{r}@{}}
\toprule
Method & \multicolumn{3}{c}{Overall} & \multicolumn{3}{c}{GPT-5.6-Terra} & \multicolumn{3}{c}{DeepSeek-V4-Flash} & \multicolumn{3}{c}{Claude-4.5-Haiku} \\
\cmidrule(lr){2-4}\cmidrule(lr){5-7}\cmidrule(lr){8-10}\cmidrule(l){11-13}
& Pass@1 & Pass@$k$ & Pass$^k$ & Pass@1 & Pass@$k$ & Pass$^k$ & Pass@1 & Pass@$k$ & Pass$^k$ & Pass@1 & Pass@$k$ & Pass$^k$ \\
\midrule
Mini-SWE-agent (Seed Harness) & 75.2 & 80.3 & 70.0 & 82.0 & 84.0 & 80.0 & 82.5 & 89.0 & 76.0 & 61.0 & 68.0 & 54.0 \\
SWE-agent & 73.3 & 79.7 & 67.0 & 74.5 & 83.0 & 66.0 & 89.5 & 93.0 & 86.0 & 56.0 & 63.0 & 49.0 \\
Codex CLI & 78.5 & 84.7 & 72.3 & 88.0 & 92.0 & 84.0 & 82.0 & 90.0 & 74.0 & 65.5 & 72.0 & 59.0 \\
\midrule
\multicolumn{13}{l}{\textit{Our developed primitives (fixed across tasks)}} \\
Change-surface Tracer & 77.7 & 83.7 & 71.7 & 79.5 & 84.0 & 75.0 & 87.5 & 92.0 & 83.0 & 66.0 & 75.0 & 57.0 \\
Compatibility-envelope Gate & 77.3 & 83.3 & 71.3 & 79.5 & 85.0 & 74.0 & 85.5 & 90.0 & 81.0 & 67.0 & 75.0 & 59.0 \\
Contract-case Explorer & 78.7 & 83.7 & 73.7 & 80.0 & 83.0 & 77.0 & 88.5 & 93.0 & 84.0 & 67.5 & 75.0 & 60.0 \\
Discriminating-oracle Runner & 76.3 & 81.0 & 71.7 & 77.5 & 82.0 & 73.0 & 86.5 & 90.0 & 83.0 & 65.0 & 71.0 & 59.0 \\
Domain-trace Template & 78.7 & 85.0 & 72.3 & 81.5 & 86.0 & 77.0 & 84.5 & 91.0 & 78.0 & 70.0 & 78.0 & 62.0 \\
Environment-capability Runner & 79.5 & 84.3 & 74.7 & 83.5 & 86.0 & 81.0 & 89.0 & 93.0 & 85.0 & 66.0 & 74.0 & 58.0 \\
\midrule
Random composer & 77.4 & 82.8 & 71.9 & 80.6 & 84.4 & 76.9 & 85.9 & 90.9 & 80.9 & 65.6 & 73.3 & 58.0 \\
\ours~(Ours) & \textbf{84.0} & \textbf{88.3} & \textbf{79.7} & \textbf{86.5} & \textbf{90.0} & \textbf{83.0} & \textbf{91.5} & \textbf{95.0} & \textbf{88.0} & \textbf{74.0} & \textbf{80.0} & \textbf{68.0} \\
\bottomrule
\end{tabular}
}
\end{table}

\textbf{\method~generalizes to out-of-domain tasks.} To test whether the primitives developed on tasks from one domain generalize to another, we sampled a subset of the primitive library where only primitives developed from other out-of-domain (OOD) tasks are retained during evaluation. For example, \ours~is evaluated on Terminal-bench but is paired with OOD primitives developed on SWE-bench. The primitives' implementation and composition contract remain unchanged while the application scope is adjusted on development tasks. As shown in Figure \ref{fig:domain_generalization}, while the Pass@1 performance of \ours~with OOD primitives underperforms that paired with in-domain (ID) primitives, its performance gain compared to the seed harness is still significant, achieving 11.1 points of Pass@1 gain when evaluated on the SWE-bench to Terminal bench setup. This generalizability of \method~is consistent with the scalability of primitive library size is shown in Figure \ref{fig:scaling}.

\begin{figure}[t]
\centering
\begin{subfigure}[t]{0.45\linewidth}
    \centering
    \includegraphics[width=\linewidth]{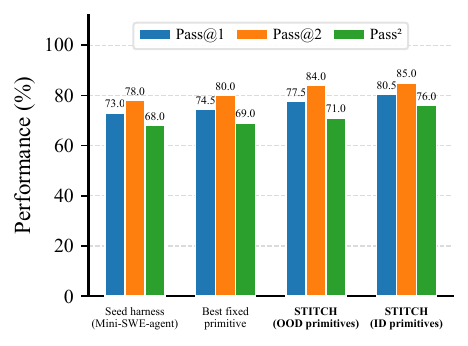}
\end{subfigure}
\begin{subfigure}[t]{0.45\linewidth}
    \centering
    \includegraphics[width=\linewidth]{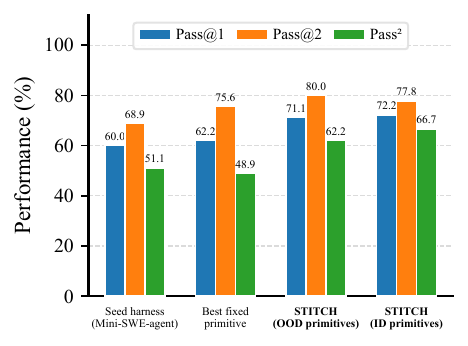}
\end{subfigure}
\caption{The cross domain generalization of \ours. \emph{Left}: The generalization of \ours~on SWE-bench with out-of-domain primitives developed on Terminal-bench. \emph{Right}: The generalization of \ours~on Terminal-bench with out-of-domain primitives developed on SWE-bench.}
\label{fig:domain_generalization}
\end{figure}

\textbf{Selected primitives activate during task execution.} We measure the frequencies of primitives that are selected and whether they are actually activated during the actor's task-solving trajectories. As shown in the left panel of Figure \ref{fig:select_cost}, while different primitives are selected at different rates, the activation rate, which measures whether the selected primitive is actually activated and executed without runtime errors on the actor agent, is stable at 100\% for all primitives. This empirically supports part 2 of Proposition \ref{thm:dominance}, demonstrating that \ours~is not only free of the exponential reliability risk of raw harness code generation, but also actually takes effect during the actor execution. We further show that the primitives are sparsely activated with the selection pattern in Appendix \ref{app:select_patn}.

\textbf{The cost of \ours~in building task-specific harnesses is minimal, down to 2.7\% of actor execution and 638$\times$ less than coding from scratch.} To quantify the cost of building task-specific harnesses, we compare \ours~with an optimistic lower bound for generating the complete harness implementation at test time. Specifically, we assume an oracle coding agent that emits valid code directly, with no input, reasoning, or debugging cost. Under this assumption, the cost is simply that of emitting the implementation's output tokens. As shown in the right panel of Figure~\ref{fig:select_cost}, \ours~costs only 4.3\% of one actor attempt on SWE-bench and 2.7\% on Terminal-Bench, whereas full-code generation costs 23.35 and 17.23 actor attempts, respectively, which is approximately 543 and 638 times the composition cost.

\begin{figure}[t]
\centering
\begin{subfigure}[t]{0.45\linewidth}
    \centering
    \includegraphics[width=\linewidth]{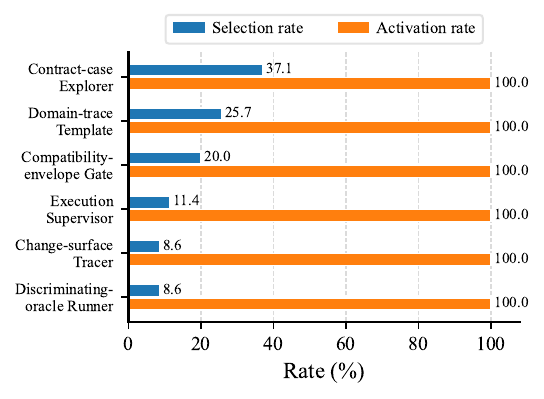}
\end{subfigure}
\begin{subfigure}[t]{0.45\linewidth}
    \centering
    \includegraphics[width=\linewidth]{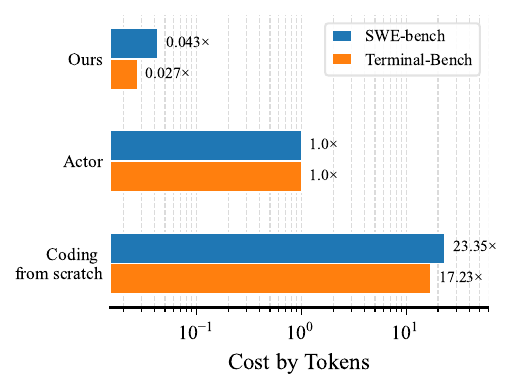}
\end{subfigure}
\caption{\emph{Left}: Despite different selection rates of primitives, all primitives are fully executable and activated at test time. \emph{Right}: The cost overhead of \ours~is as low as 2.7\% compared to the actor. In comparison, even the lower bound cost of coding the harness from scratch is 17.23$\times$ higher than the actor cost, making \ours~638$\times$ efficient.}
\label{fig:select_cost}
\end{figure}

\section{Conclusion}

We introduced \method, reusable harness mechanisms developed from recurring execution failures and equipped with application scopes and composition contracts. Based on these primitives, \ours~constructs task-specific harnesses by separating task-conditioned selection from deterministic compilation. Our analysis examines the task-mechanism mismatch induced by fixed choices and the execution risks of generating mechanism implementations at test time. Evaluation on held-out subsets of SWE-bench Verified and Terminal-Bench 2 shows improvements over the seed harness and surpasses human-designed harnesses such as Codex CLI, with composition costs of 4.3\% and 2.7\% of one actor attempt, respectively. The experiments also examine library expansion and transfer across actor models and task domains. Together, these results support reusable mechanisms and task-conditioned composition as an approach to task-specific harness construction.

\section*{AI use statement}

In preparing this manuscript, we used AI tools (Codex) only to polish the language and improve the clarity of the presentation. Their use was limited to correcting grammar, rephrasing author-written sentences, and improving readability and consistency of wording. The research questions, core ideas, methodology, experimental design, analyses, and scientific conclusions of this work were developed by the human authors, and no claims, results, or references were generated by these tools. Every AI-assisted edit was reviewed for correctness by at least two human authors. We take full responsibility for the final content of this work.

\bibliography{iclr2027_conference}
\bibliographystyle{iclr2027_conference}

\newpage
\appendix

\section{Proofs of Theoretical Results}

\subsection{Proof of Proposition \ref{thm:suboptimal}}

\begin{proposition*}[Fundamental Suboptimality of Monolithic Harnesses, Restated]
Assume $\mu_i \in (0,1)$, $\bar{u}_i > 0$, and $\bar{c}_i > 0$ for all $i \in \{1, \dots, M\}$. The expected performance of the optimal task-independent harness, $\mathbf{a}_{\text{fixed}}^* = \arg\max_{\mathbf{a} \in \{0, 1\}^M} \mathbb{E}_{x \sim \mathcal{D}}[R(x, \mathbf{a})]$, is strictly lower than that of the oracle task-adaptive harness $\mathbf{a}^*(x)$. The suboptimality gap $\Gamma_{\text{fixed}}$ is defined as:
\begin{equation*}
    \Gamma_{\text{fixed}} = \mathbb{E}_{x}[R(x, \mathbf{a}^*(x))] - \mathbb{E}_{x}[R(x, \mathbf{a}_{\text{fixed}}^*)] = \sum_{i=1}^M \min\left( \mu_i \bar{u}_i, \; (1 - \mu_i) \bar{c}_i \right) > 0
\end{equation*}
\end{proposition*}

\begin{proof}
Under the pointwise sign assumptions, one oracle activation maximizing $R(x,\mathbf{a})$ for each task is:
\begin{equation}
    a_i^*(x)=\mathbf{1}(x\in\Omega_i).
\end{equation}
The expected performance of the oracle adaptive harness is:
\begin{equation}
    \mathbb{E}_{x \sim \mathcal{D}}[R(x, \mathbf{a}^*(x))] = \mathbb{E}[R_0(x)] + \sum_{i=1}^M \mathbb{E}_{x}[\mathbf{1}(x \in \Omega_i) u_i(x)] = \mathbb{E}[R_0(x)] + \sum_{i=1}^M \mu_i \bar{u}_i
\end{equation}

For a fixed harness $\mathbf{a}_{\text{fixed}} \in \{0, 1\}^M$, the expected performance over $\mathcal{D}$ is:
\begin{equation}
    \mathbb{E}_{x \sim \mathcal{D}}[R(x, \mathbf{a}_{\text{fixed}})] = \mathbb{E}[R_0(x)] + \sum_{i=1}^M a_i \mathbb{E}_{x \sim \mathcal{D}}[\Delta_i(x)]
\end{equation}
Using the law of total expectation:
\begin{equation}
    \mathbb{E}_{x}[\Delta_i(x)] = \mathbb{P}(x \in \Omega_i)\mathbb{E}[u_i(x) \mid x \in \Omega_i] - \mathbb{P}(x \notin \Omega_i)\mathbb{E}[c_i(x) \mid x \notin \Omega_i] = \mu_i \bar{u}_i - (1 - \mu_i) \bar{c}_i
\end{equation}

To maximize this sum, the optimal fixed harness sets:
\begin{equation}
    a_{i, \text{fixed}}^* = \begin{cases} 1, & \text{if } \mu_i \bar{u}_i > (1 - \mu_i) \bar{c}_i \\ 0, & \text{otherwise} \end{cases}
\end{equation}
The optimal fixed expectation is therefore:
\begin{equation}
    \mathbb{E}_{x \sim \mathcal{D}}[R(x, \mathbf{a}_{\text{fixed}}^*)] = \mathbb{E}[R_0(x)] + \sum_{i=1}^M \max\left( 0, \; \mu_i \bar{u}_i - (1 - \mu_i) \bar{c}_i \right)
\end{equation}

Subtracting $\mathbb{E}_{x}[R(x, \mathbf{a}_{\text{fixed}}^*)]$ from $\mathbb{E}_{x}[R(x, \mathbf{a}^*(x))]$ yields the gap $\Gamma_{\text{fixed}}$:
\begin{equation}
    \Gamma_{\text{fixed}} = \sum_{i=1}^M \left( \mu_i \bar{u}_i - \max\left( 0, \; \mu_i \bar{u}_i - (1 - \mu_i) \bar{c}_i \right) \right)
\end{equation}
Using the identity $y - \max(0, y - z) = \min(y, z)$ for $y = \mu_i \bar{u}_i$ and $z = (1 - \mu_i) \bar{c}_i$:
\begin{equation}
    \Gamma_{\text{fixed}} = \sum_{i=1}^M \min\left( \mu_i \bar{u}_i, \; (1 - \mu_i) \bar{c}_i \right)
\end{equation}
Because $\mu_i \in (0, 1)$, $\bar{u}_i > 0$, and $\bar{c}_i > 0$, each term in the sum is strictly positive, proving $\Gamma_{\text{fixed}} > 0$.
\end{proof}

\subsection{Proof of Proposition \ref{thm:dominance}}

\begin{proposition*}[Conditional Dominance of Primitive-Based Composition, Restated]
Under the model in Section~\ref{sec:theory}, let $\bar{k}=\mathbb{E}_x[k(x)]=\sum_{i=1}^M\mu_i$. Primitive-based task-adaptive composition outperforms the specified alternatives under the following conditions:
\begin{enumerate}
\item \textbf{Dominance over Static Harnesses:} $\mathbb{E}[J(x,H_{\text{comp}})] >\mathbb{E}[J(x,H_{\text{fixed}}^*)]$ if and only if its utility-weighted selection loss is smaller than the fixed mismatch gap:
    \begin{equation}
        \sum_{i=1}^M
        \left[
            \mu_i\beta_i\bar{u}_i
            +(1-\mu_i)\alpha_i\bar{c}_i
        \right]
        <\Gamma_{\text{fixed}}.
    \end{equation}

\item \textbf{Dominance over Raw Code Generation:} For this part, assume $k(x)=\bar{k}$ for every task, although the selected mechanisms may differ across tasks. Suppose valid generated harnesses attain the oracle performance $R(x,\mathbf{a}^*(x))$. If $\mathbb{E}[J(x,H_{\text{comp}})]>0$, composition strictly outperforms this oracle generator if and only if
    \begin{equation}
        \bar{k}>k^*
        =
        \frac{
            \ln\left(
                \frac{\mathbb{E}[R(x,\mathbf{a}^*(x))]}
                     {\mathbb{E}[J(x,H_{\text{comp}})]}
            \right)
        }{-\ln(1-\epsilon)}.
    \end{equation}
For small $\epsilon$, the threshold satisfies
    \begin{equation}
        k^*\approx
        \frac{1}{\epsilon}
        \ln\left(
            \frac{\mathbb{E}[R(x,\mathbf{a}^*(x))]}
                 {\mathbb{E}[J(x,H_{\text{comp}})]}
        \right).
    \end{equation}
\end{enumerate}
\end{proposition*}

\begin{proof}
Write $A=\mathbb{E}_x[R(x,\mathbf{a}^*(x))]$ and $B=\mathbb{E}[J(x,H_{\text{comp}})]$.

\textbf{Part 1: Dominance over Static Harnesses.} By the definitions of the utility-weighted error rates,
\begin{align}
    \mathbb{E}[a_i\Delta_i(x)]
    &=
    \mu_i\mathbb{E}[a_i u_i(x)\mid x\in\Omega_i]
    -(1-\mu_i)\mathbb{E}[a_i c_i(x)\mid x\notin\Omega_i]
    \nonumber\\
    &=
    \mu_i(1-\beta_i)\bar{u}_i
    -(1-\mu_i)\alpha_i\bar{c}_i.
\end{align}
Using the additive model and the assumed validity of composition,
\begin{equation}
    B=A-\sum_{i=1}^M
    \left[
        \mu_i\beta_i\bar{u}_i
        +(1-\mu_i)\alpha_i\bar{c}_i
    \right].
    \label{eq:comp_utility}
\end{equation}
Proposition~\ref{thm:suboptimal} gives $\mathbb{E}[J(x,H_{\text{fixed}}^*)]=A-\Gamma_{\text{fixed}}$. Comparing these expressions establishes the first condition.

\textbf{Part 2: Dominance over Raw Code Generation.} Under the conditional failure model and oracle-performance assumption,
\begin{equation}
    \mathbb{E}[J(x,H_{\text{gen}})]
    =
    \mathbb{E}_x\left[
        (1-\epsilon)^{k(x)}R(x,\mathbf{a}^*(x))
    \right]
    =
    (1-\epsilon)^{\bar{k}}A,
\end{equation}
where the last equality uses $k(x)=\bar{k}$ for every task. Since $0<B\leq A$,
\begin{equation}
    B>(1-\epsilon)^{\bar{k}}A
    \quad\Longleftrightarrow\quad
    \bar{k}>
    \frac{\ln(A/B)}{-\ln(1-\epsilon)}.
\end{equation}
The approximation follows from $-\ln(1-\epsilon)=\epsilon+O(\epsilon^2)$. For varying mechanism counts, the exact expectation must be retained; replacing $k(x)$ by its mean does not generally yield the required upper bound.
\end{proof}

\section{An Illustrative Primitive}
\label{app:primitive-example}

Figure~\ref{fig:state-guard-example} illustrates how a recurring failure in development trajectories motivates a reusable primitive, and how its scope and contract guide task-specific composition. State Guard preserves task-critical source files; the actor remains responsible for solving the task.

\begin{figure}[htbp]
\centering
\begin{tcolorbox}[
  colback=blue!2!white,
  colframe=blue!45!black,
  colbacktitle=blue!45!black,
  coltitle=white,
  title={State Guard: preserving task-critical inputs},
  fonttitle=\bfseries,
  boxrule=0.5pt, arc=1.5mm,
  left=7pt, right=7pt, top=5pt, bottom=5pt
]
\small
\setlength{\tabcolsep}{0pt}
\renewcommand{\arraystretch}{1.13}
\begin{tabularx}{\linewidth}{@{}>{\bfseries\raggedright\arraybackslash}p{0.19\linewidth}@{\hspace{9pt}}>{\raggedright\arraybackslash}X@{}}
\multicolumn{2}{@{}l@{}}{\textbf{From development failures to a reusable mechanism}} \\
\midrule
Failure cluster & Actor actions destroy inputs needed later: opening a database removes recovery evidence, or modifying a reference program invalidates subsequent comparisons. These failures share a missing source-preservation mechanism. \\
\midrule
Proposal & Preserve designated inputs across actor actions and report attempted modifications, so exploration does not irreversibly lose essential evidence. \\
\midrule
Implementation & Save original file contents and hashes before execution. Block recognized direct writes; check for indirect changes after each action. Restore changed files and return a failed-action observation naming the affected paths. \\
\midrule
Application scope & Select when local inputs are irreplaceable, actions plausibly risk modifying them, and outputs are separate. Exclude files that the task requires editing in place; merely having reference files is insufficient. \\
\midrule
Composition contract & \textit{Inputs:} source paths to protect and declared output paths. \textit{Outputs:} mutation reports and integrity-check results. \textit{Setup:} file access and storage for recoverable copies. \textit{Dependencies:} snapshot before the first actor action; check after each action; deliver violation feedback to the next actor request; expose the final integrity check to completion logic. \\
\midrule
\multicolumn{2}{@{}l@{}}{\textbf{Illustrative application: database recovery}} \\
\midrule
Task & Recover database records into a separate JSON file. The database and its write-ahead log (WAL), a companion file containing recovery evidence, must remain available. \\
\midrule
Seed-harness failure & The actor opens the original database before preserving its WAL. The operation removes recovery evidence, leaving the actor to guess missing records. Both recorded seed-harness attempts exhibited this failure. \\
\midrule
Guarded execution & The harness first snapshots the database and WAL. If an action changes or deletes either, State Guard restores the original contents and reports the action as failed. The feedback enables the actor to change its approach, for example by operating on a working copy while preserving the originals. \\
\midrule
Validation & Tests cover direct and indirect mutations while allowing output creation. In two development attempts on this task, the guard activated and both recovery outputs passed evaluation. Preservation alone does not establish that a recovered answer is correct. \\
\bottomrule
\end{tabularx}
\end{tcolorbox}
\caption{\textbf{An illustrative State Guard primitive.} The same file-preservation mechanism applies to different task-critical inputs without implementing a domain-specific solver. The case summarizes development records rather than reproducing a verbatim trajectory; it is not a controlled estimate of improvement or evidence of held-out generalization. Protection covers designated local file contents, not arbitrary external side effects.}
\label{fig:state-guard-example}
\end{figure}

\section{An Example of Composition Intent}
\label{app:composition-intent}

A composition intent specifies the primitives to use and their configuration; the compiler resolves their dependencies and constructs the executable graph. Figure~\ref{fig:composition-intent-example} illustrates the interface for a repository-repair task that needs repository context and verification feedback. The example follows the implemented intent schema and the compilation procedure in Section~\ref{sec:composition}.

\begin{figure}[htbp]
\centering
\begin{tcolorbox}[
 colback=blue!2!white,colframe=blue!45!black,
 colbacktitle=blue!45!black,coltitle=white,
 title={Composition intent: repository context and verified repair},
 fonttitle=\bfseries,boxrule=0.5pt,arc=1.5mm,
 left=7pt,right=7pt,top=5pt,bottom=5pt]
\small
\setlength{\tabcolsep}{0pt}
\renewcommand{\arraystretch}{1.12}
\begin{tabularx}{\linewidth}{@{}>{\raggedright\arraybackslash}p{0.30\linewidth}@{\hspace{9pt}}>{\raggedright\arraybackslash}X@{}}
\multicolumn{2}{@{}l@{}}{\textbf{Schema fields}}\\
\midrule
\texttt{strategy\_id}, \texttt{rationale} & Strings naming the composition and explaining the selection.\\
\texttt{primitives} & A list of instances, each with a unique local \texttt{alias}, a registered \texttt{primitive\_id}, and a \texttt{params} object.\\
\bottomrule
\end{tabularx}
\medskip
\textbf{Example intent (JSON)}
\begin{verbatim}
{
  "strategy_id": "repository_repair",
  "rationale": "Provide context and repair failed checks.",
  "primitives": [
    {
      "alias": "repository_context",
      "primitive_id": "build_repository_map",
      "params": {}
    },
    {
      "alias": "verification_feedback",
      "primitive_id": "repair_from_verifier_feedback",
      "params": {}
    }
  ]
}
\end{verbatim}
\end{tcolorbox}
\caption{\textbf{An illustrative composition intent.} The composer names reusable operations and explains their purpose; the deterministic compiler supplies the required dependencies and graph connections. Aliases identify instances within the intent, so the composer does not need to assign runtime node identifiers.}
\label{fig:composition-intent-example}
\end{figure}

\section{Harness Composer Prompt}
\label{app:composer-prompt}

Figure~\ref{fig:composer-prompt} presents the prompt used for initial harness composition. The composer receives the task, execution specification, seed graph, primitive catalog, and output schema before the actor starts. Its response wraps the composition intent illustrated in Figure~\ref{fig:composition-intent-example} with a Boolean decision indicating whether to change the seed graph and a short explanation.

\begin{figure}[htbp]
\centering
\begin{tcolorbox}[
 colback=blue!2!white,colframe=blue!45!black,
 colbacktitle=blue!45!black,coltitle=white,
 title={Composer prompt: selecting a task-specific harness},
 fonttitle=\bfseries,boxrule=0.5pt,arc=1.5mm,
 left=7pt,right=7pt,top=5pt,bottom=5pt]
\small
\textbf{System message}
\medskip

Using the passive task context, public issue statement, and primitive catalog, apply each primitive's choose\_only\_if criteria literally. Use repository and version only to interpret architecture or compatibility; never select a primitive merely because of repository identity. Use seed harness when no other primitive fully qualifies. If several focused primitives qualify but are incompatible, select the one whose evidence is needed earliest before editing. Do not solve the issue or predict difficulty. Return only schema-valid JSON with a brief causal rationale.

\medskip
\textbf{User message template}
\begin{verbatim}
OUTPUT_SCHEMA:
<JSON schema for GraphPlannerIntentProposal>
PLANNING_INPUT:
{
  "planning_mode": "initial_single_pass",
  "task": <task instruction and metadata>,
  "task_profile": <task profile, or null>,
  "environment_capabilities": <available capabilities>,
  "actor_model": <actor model and interaction interface>,
  "budget": <execution limits and remaining budget>,
  "active_graph": <compact seed-graph description>,
  "audit": <initial composition objective>,
  "primitive_catalog": <available primitive descriptions>,
  "intent_contract": <alias, primitive, rules>,
}
\end{verbatim}
\textbf{Expected response}
\medskip

Return \texttt{should\_mutate}, \texttt{rationale}, and \texttt{intent}. Set \texttt{should\_mutate=true} and provide a composition intent when a change is appropriate. Set \texttt{should\_mutate=false} and \texttt{intent=null} to retain the seed graph. A supplied intent follows the schema in Figure~\ref{fig:composition-intent-example}; it selects primitive instances and leaves mandatory dependency wiring to the compiler.
\end{tcolorbox}
\caption{Example harness-composer prompt.}
\label{fig:composer-prompt}
\end{figure}

\section{Implementation Details of Primitive Libraries and Compiler}
\label{app:primitive-catalog}

Tables~\ref{tab:app-swe-primitives} and~\ref{tab:app-terminal-primitives} introduce the primitives presented in Tables~\ref{tab:swe-luna} and~\ref{tab:terminal-random-test}, respectively. Each entry describes an application scope and the information or execution result that helps the actor choose its next action. The actor remains responsible for completing the task.

For the deterministic compiler, we note that each fixed-primitive baseline evaluates a fixed harness configuration containing one designated focused control mechanism, its required dependencies, and shared supporting operations, forming a minimal compatible set of operations handled by the compiler instead of an isolated primitive. The compiler also handles the integration of multiple primitives when multiple primitives are selected in the composition intent.

\section{Composer Selection Patterns}
\label{app:select_patn}

With the implementation details of \ours~detailed in Appendix \ref{app:primitive-catalog}, we additionally provide the primitive selection pattern of the composer in \ours~. As shown in Figure \ref{fig:select_cnt}, the composer selects 0 primitives, that is, the seed harness, for about half of the tasks given the readily competent performance of the seed harness Mini-SWE-agent. For the other half of the tasks, the majority of the tasks selects only a single primitive, with a few cases selecting more than one primitive. The combinations of these multi-primitive selections are shown in Figure \ref{fig:select_comb}. This relatively sparse selection pattern shows that \ours~follows the idea of applying harness mechanisms selectively instead of using a fixed harness and dense activations of all the mechanisms.

\begin{figure}[t]
    \centering
    \includegraphics[width=0.45\linewidth]{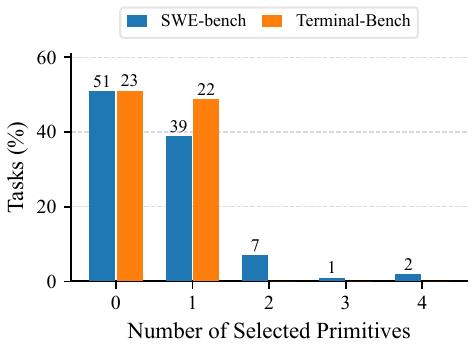}
    \vspace{-5mm}
\caption{The distribution of the selected primitives by the composer.}

    \label{fig:select_cnt}
\end{figure}

\begin{figure}[t]
    \centering
    \includegraphics[width=1\linewidth]{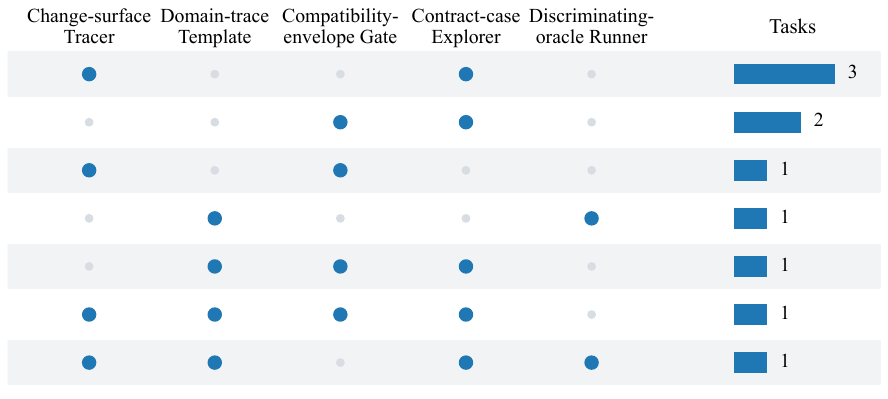}
    \vspace{-5mm}
\caption{The combination of multi-primitive selections by the composer.}

    \label{fig:select_comb}
\end{figure}

\begin{table}[htbp]
\centering\small
\renewcommand{\arraystretch}{1.18}
\setlength{\tabcolsep}{5pt}
\caption{The six SWE-bench primitives introduced in the main paper. Application scopes guide selection; the described checks and evidence guide repository inspection, editing, and verification.}
\label{tab:app-swe-primitives}
\begin{tabularx}{\linewidth}{@{}>{\raggedright\arraybackslash}p{0.22\linewidth}>{\raggedright\arraybackslash}p{0.29\linewidth}>{\raggedright\arraybackslash}X@{}}
\toprule
\textbf{Primitive} & \textbf{Application scope} & \textbf{Operation and use of its output}\\
\midrule
Change-surface Tracer & One public behavior spans coupled implementations or representations that must obey the same rule. & Identifies the behavior owner and an alternate implementation path to inspect. The resulting checks direct attention to related edits or tests that a single-file repair could miss.\\
\addlinespace[5pt]
Compatibility-envelope Gate & A requested change must preserve a concrete legacy input, entry point, or previously supported behavior. & Pairs the requested behavior with preservation and alternate legacy checks. These checks make the compatibility requirement explicit when the actor chooses and verifies a patch.\\
\addlinespace[5pt]
Contract-case Explorer & A nearby case or an interaction between features can expose an incomplete interpretation of the requested behavior. & Develops a reported case, a preservation control, and a contrasting alternate case. Their expected observations help distinguish a narrow patch from a repair covering the behavioral boundary.\\
\addlinespace[5pt]
Discriminating-oracle Runner & An ordinary return-value or exception check could pass while a relevant side effect remains wrong. & Specifies a focused observation of behavior such as mutation, object identity, callback order, or generated structure, together with a preservation control. This gives the actor a more discriminating correctness check.\\
\addlinespace[5pt]
Domain-trace Template & A value or reference crosses changes in meaning, ownership, or lifecycle state. & Organizes checks around the relevant transitions and an invariant that should survive them. The trace guides the actor toward the point where the value first acquires the wrong meaning or state.\\
\addlinespace[5pt]
Environment-capability Runner & A missing executable, dependency, backend, locale, or operating-system facility may prevent the reported behavior from being exercised. & Specifies a capability probe alongside behavior and preservation checks. The probe helps the actor distinguish unavailable environment support from an observed defect before interpreting a test result.\\
\bottomrule
\end{tabularx}
\end{table}

\begin{table}[htbp]
\centering\small
\renewcommand{\arraystretch}{1.18}
\setlength{\tabcolsep}{5pt}
\caption{The five Terminal-Bench primitives introduced in the main paper. Each operation uses inputs available from the task or ordinary actor execution and returns evidence for the next action or completion check.}
\label{tab:app-terminal-primitives}
\begin{tabularx}{\linewidth}{@{}>{\raggedright\arraybackslash}p{0.22\linewidth}>{\raggedright\arraybackslash}p{0.29\linewidth}>{\raggedright\arraybackslash}X@{}}
\toprule
\textbf{Primitive} & \textbf{Application scope} & \textbf{Operation and use of its output}\\
\midrule
State Guard & Irreplaceable source files may be modified by task commands, while the required output is stored separately. & Saves source contents and hashes, detects changes, and restores modified files. Reports the affected paths so the actor can work on a copy or a separate output; checks source integrity before completion.\\
\addlinespace[5pt]
Check Replay & A concrete executable check is available and should remain valid as the actor revises the deliverable. & Retains the check command and its execution context, then reruns the same check at completion. Records the result and restores local workspace writes made by the check, exposing regressions without weakening the acceptance criterion.\\
\addlinespace[5pt]
First-Failure Localizer & Reference and candidate commands produce comparable ordered observations whose first difference can identify a repair target. & Runs both commands from the same initial workspace state and compares exit status, standard output, and error output. Returns the matching prefix and first difference to focus the actor's next investigation.\\
\addlinespace[5pt]
Experiment Keeper & Iterative search produces candidate files that can be compared using a numeric objective with a known optimization direction. & Repeatedly measures valid candidates, compares their median scores, and retains the best candidate's contents. Restores the retained candidate after a worse trial and remeasures it before completion.\\
\addlinespace[5pt]
Execution Supervisor & A background job or persistent service needs explicit lifecycle management, including a readiness check when appropriate. & Manages the process group, deadline, logs, and cleanup. Records job completion or probes service readiness, allowing the actor to proceed using an observed execution result rather than repeated manual launch-and-poll steps.\\
\bottomrule
\end{tabularx}
\end{table}

\end{document}